%% file: main.tex
\documentclass[10pt]{article}
\usepackage[preprint]{tmlr}  %
\input{math_commands.tex}

\usepackage{hyperref}
\usepackage{url}
\usepackage{amsmath,amssymb,amsthm}
\usepackage{graphicx}
\usepackage{booktabs}
\usepackage{multirow}
\usepackage{caption}
\usepackage{subcaption}
\usepackage{pifont}
\usepackage{wrapfig}
\usepackage{xspace}

\newcommand{\cmark}{{\normalfont\ding{51}}}
\newcommand{\xmark}{{\normalfont\ding{55}}}
\newtheorem{theorem}{Theorem}
\newtheorem{lemma}{Lemma}
\newtheorem{proposition}{Proposition}

\title{Revisiting Euler-Angle Regression with Kolmogorov-Arnold Networks}

\author{\name Yangting Sun \email sunyt1230@gmail.com \\
      \addr Independent Researcher
      \AND
      \name Zijun Cui \email cuizijun@msu.edu \\
      \addr Michigan State University
      \AND
      \name Yufei Zhang \email yufeizhang96@outlook.com \\
      \addr Independent Researcher}

\begin{document}
\maketitle

\begin{abstract}
In many real-world systems, including articulated robots and biomechanical models, rotations are defined in joint space and naturally parameterized by Euler angles with bounded ranges. Yet regressing Euler angles remains challenging, as their discontinuities and singularities often destabilize training. In this work, we revisit Euler-angle regression and show that its effectiveness depends critically on the interaction between rotation representation, regression architecture, and domain constraints. We introduce a new framework that combines range-aware Euler modeling with Kolmogorov-Arnold Networks (KAN), which replace fixed node-wise activations with learnable univariate functions on edges. We further provide theoretical analysis indicating that bounded Euler ranges motivate a near-additive structure in the regression function, which favors the additive functional form of KAN, and we confirm this trend empirically. Extensive experiments on controlled rotation regression, object pose estimation, and robotic and human inverse kinematics demonstrate consistent improvements in accuracy, convergence, and efficiency. The code will be publicly available.
\end{abstract}

\section{Introduction}
\label{sec:intro}

Rotation describes the orientation of rigid bodies, cameras, and articulated parts. Estimating rotations from sensor data is fundamental across computer vision, graphics, robotics, and biomechanics, with applications in camera localization \cite{zhang2024cameras}, human kinematics and dynamics estimation \cite{le2026quamo,xia2025reconstructing,zhang2025diffusion,Zhang_2024_CVPR,ismayilzada2026pad,delp2007opensim}, and robotic manipulation \cite{spong2008robot}. With the rise of end-to-end deep learning, rotation estimation is now commonly approached by training neural networks to directly regress rotations from sensor measurements~\cite{levinson2020analysis,peretroukhin2020smooth,gu2023learning,okorn2023deep}.

3D rotation regression is fundamentally challenging because rotations lie on $\mathrm{SO}(3)$, a curved, non-Euclidean manifold that aligns poorly with conventional deep learning frameworks~\cite{bronstein2017geometric}. Specifically, classical rotation representations use minimal parameterizations that map $\mathbb{R}^3$ to $\mathrm{SO}(3)$. While these classical parameterizations offer clear interpretability in task-space~\cite{schuck2025primer,sfikas2025euler} or respect the Lie group structure of $\mathrm{SO}(3)$~\cite{grassia1998practical}, they are subject to discontinuities and singularities that challenge the standard learning process. For example, Euler angles capture sequential rotations about three Euclidean axes, but exhibit discontinuities from angular periodicity and singular configurations (commonly known as gimbal lock) where the mapping loses local invertibility.

Rather than adopting minimal parameterizations as in classical approaches, recent works propose overparameterized representations of $\mathrm{SO}(3)$ for rotation regression outputs~\cite{bregier2021deep,peretroukhin2020smooth}. Notably, the 6D representation~\cite{zhou2019continuity} parameterizes a 3D rotation using two unconstrained 3D vectors. While this strategy incurs representational redundancy and an additional post-orthogonalization step, it yields a continuous embedding and mitigates singularities. As a result, the 6D representation has become widely adopted in various pipelines, with empirical results reported in tasks such as monocular 3D human body and hand reconstruction~\cite{xia2025reconstructing,yu2025dyn} and humanoid control~\cite{Luo_2023_ICCV}.

Despite the widespread usage of the 6D representation, two aspects remain insufficiently examined in existing methods. First, prior work has largely focused on representation design while keeping the regression architecture fixed, most commonly using multilayer perceptrons (MLP)~\cite{zhou2019continuity}. The interplay between rotation parameterization and network architecture is underexplored. Second, in real-world systems, rotations are often modeled using Euler angles with application-specific range constraints that can be particularly tight. For example, the human body and hand exhibit anatomically meaningful joint rotations with constrained degrees of freedom and strict range limits~\cite{zhang2023body,zhang2024weakly,lin2026monocular}. Similarly, robotic systems such as the Franka arm~\cite{bensadoun2022neural} and Allegro hand~\cite{chen2024object} are controlled directly in joint space, where actuation, torque, and safety constraints are defined over bounded rotational coordinates. How to faithfully integrate such range-constrained Euler angles in a learning-based framework has received limited study.

In this work, we revisit Euler-angle regression by addressing these two gaps. We target the bounded-range, application-native regime, rather than generic $\mathrm{SO}(3)$ prediction. Instead of relying on redundant parameterizations and post-orthogonalization steps to enforce $\mathrm{SO}(3)$ constraints, we operate directly in Euler-angle space. Recognizing the intrinsic angular structure of Euler angles, we introduce Kolmogorov-Arnold Networks (KAN)~\cite{liu2024kan} as a functionally adaptive regression backbone. Unlike conventional MLP, KAN replaces fixed node-wise activations with learnable univariate functions defined on edges, letting each channel adapt its nonlinearity to the structure of bounded Euler-angle targets. Furthermore, we exploit the practical range constraints of Euler angles to mitigate discontinuities and singular behavior. We present theoretical analysis showing how these design choices collectively establish an improved regression framework. Through extensive evaluation in controlled settings and real-world applications, including object pose estimation, a robotic manipulator, and biomechanical hand and full-body models, we validate our proposed framework.

In summary, our primary contributions are as follows:
\begin{itemize}
\item To the best of our knowledge, this is the first work that introduces KAN for 3D rotation regression, while leveraging the intrinsic bounded and angular structure of the Euler-angle coordinates.
\item We incorporate practical Euler-angle range limits and a constraint-aware axis ordering to mitigate discontinuities and singularities, accompanied by a theoretical analysis of their coupling with KAN.
\item Through extensive experiments spanning controlled rotation regression, object pose estimation, and robotic and human inverse kinematics, we show that the proposed framework improves accuracy, efficiency, and optimization stability over baseline methods.
\end{itemize}

\section{Related Work}
\label{sec:relatedwork}

\noindent\textbf{Learning-Based Rotation Regression.} When learning over the non-Euclidean manifold $\mathrm{SO}(3)$, a dominant line of work attributes optimization instability to the topological limitations of minimal parameterizations. \citet{zhou2019continuity} discussed the discontinuity and singularity issues in minimal representations and proposed the 6D rotation representation constructed from two rotation-matrix columns followed by orthonormalization. Subsequent studies further explored how to project overparameterized network outputs onto valid rotations, including SVD-based orthogonalization~\cite{bregier2021deep,levinson2020analysis}, and approaches that regress full rotation matrices without explicit orthogonalization to improve training efficiency~\cite{gu2023learning}. More recent analyses also examine how representation choice affects optimization and learning behavior~\cite{geist2024learning,schuck2025primer}. Current work on rotation regression primarily focuses on formulating various output spaces, while ignoring the valuable joint-angle constraints inherent in real-world articulated systems, as discussed below.

\noindent\textbf{Euler Angles in Practical Articulated Systems.} Euler angles explicitly parameterize degrees of freedom and encode joint range constraints in real-world articulated systems. Early learning-based methods incorporate this structure by directly regressing Euler angles in tasks such as video-based head pose and body pose estimation~\cite{borghi2018face,habermann2020deepcap}. The model performance of this direct formulation is suboptimal due to discontinuities and singularities. Other methods instead introduce structured rotation decompositions derived from kinematic priors, such as swing-twist formulations for human body and hand modeling~\cite{villegas2018neural,li2021hybrik}. Recent works predict 6D rotations and subsequently convert them to Euler angles, either during training~\cite{zhang2024weakly,zhang2023body} or via offline labeling~\cite{xia2025reconstructing}, to enforce joint-angle constraints. These practices underscore the reliance on Euler angles. However, they do not leverage the structural constraints of Euler angles as a valuable inductive bias for rotation regression directly during learning.

\noindent\textbf{Neural Networks with Inductive Bias.}
Incorporating structural inductive biases into neural networks has been shown to improve model performance and optimization stability~\cite{cui2023knowledge}. For example, geometric equivariant networks encode symmetries directly into the architecture~\cite{deng2021vector,melnyk2026role}, leading to improved sample efficiency, generalization, and robustness to geometric variations. Similarly, physical principles have inspired architectures such as Hamiltonian neural networks~\cite{greydanus2019hamiltonian} and physics-guided transformer world models~\cite{liu2026kepler}. Such structural priors benefit a range of tasks, including human motion prediction~\cite{zhang2024incorporating}. Particularly, Kolmogorov-Arnold Networks (KAN)~\cite{liu2024kan} introduce a distinct architectural paradigm inspired by the Kolmogorov-Arnold representation theorem. KAN has shown promise in function approximation and symbolic regression, yet its advantages for 3D rotation regression have not been established. To our knowledge, we are the first to demonstrate its benefits for this setting. Moreover, our approach differs from prior work by explicitly leveraging the highly restricted joint-angle ranges in practical systems to mitigate discontinuities and singularities during the prediction of Euler angles.

\section{Methodology}
\label{sec:method}

We focus on Euler-angle regression under practical range constraints. We introduce a new framework that effectively leverages KAN while incorporating the constrained, angular structure. We first review rotation representations and Euler-angle parameterizations in Section~\ref{subsec:preliminaries}. We then formalize the 3D rotation regression problem in Section~\ref{subsec:problem_formulation}. Next, we introduce KAN and its adaptation to Euler-based rotation regression in Section~\ref{subsec:kan_euler}. Finally, we present practical strategies to mitigate discontinuities and singularities in Euler-angle regression in Section~\ref{subsec:practical_guidelines}, and provide theoretical analysis in Section~\ref{subsec:theoretical_analysis} to explain why the proposed design improves regression behavior.
\subsection{Preliminaries}
\label{subsec:preliminaries}

A 3D rotation is an element of the special orthogonal group
\begin{equation}
\mathrm{SO}(3)
=
\{ R \in \mathbb{R}^{3 \times 3} \mid R^\top R = I,\ \det(R)=1 \}.
\end{equation}
Different parameterizations of $R$ provide different coordinate charts of $\mathrm{SO}(3)$. We review the studied Euler-angle parameterization below.

\noindent\textbf{Euler Angles.}
Euler angles describe a rotation as a sequence of three elemental rotations about coordinate axes, and the resulting parameterization is order dependent. For example, under a $ZXY$ convention with Euler angles
$\boldsymbol{\theta} = [\alpha,\beta,\gamma]^\top$,
the rotation matrix is
\begin{equation}
R(\alpha,\beta,\gamma)
=
R_z(\alpha)\,R_x(\beta)\,R_y(\gamma),
\end{equation}
where we denote $c_\alpha=\cos\alpha$, $c_\beta=\cos\beta$, $c_\gamma=\cos\gamma$,
and $s_\alpha=\sin\alpha$, $s_\beta=\sin\beta$, $s_\gamma=\sin\gamma$. The elemental rotations are
\begin{equation}
R_z(\alpha)=
\begin{bmatrix}
c_\alpha&-s_\alpha&0\\
s_\alpha&c_\alpha&0\\
0&0&1
\end{bmatrix},\quad
R_x(\beta)=
\begin{bmatrix}
1&0&0\\
0&c_\beta&-s_\beta\\
0&s_\beta&c_\beta
\end{bmatrix},\quad
R_y(\gamma)=
\begin{bmatrix}
c_\gamma&0&s_\gamma\\
0&1&0\\
-s_\gamma&0&c_\gamma
\end{bmatrix}.
\end{equation}

Euler angles are highly interpretable, as each parameter represents a rotation about a fixed axis and the rotation matrix has an explicit trigonometric structure in $(\alpha,\beta,\gamma)$. 
However, singularities occur at $\beta=\pm\frac{\pi}{2}$. 
For instance, when $\beta=\frac{\pi}{2}$, the rotation matrix depends only on $\alpha+\gamma$, so $(\alpha,\tfrac{\pi}{2},\gamma)$ and $(\alpha+\delta,\tfrac{\pi}{2},\gamma-\delta)$ yield the same rotation. 
Near this configuration, small perturbations in $R$ can induce large variations in $\alpha$ and $\gamma$.

\subsection{Problem Formulation}
\label{subsec:problem_formulation}
Different rotation representations have been studied across various downstream tasks. Here we consider real-world systems that often operate directly in Euler angles. Let $\mathbf{x} \in \mathcal{X}$ denote the observational input associated with a rotation, such as kinematic descriptors. Our objective is therefore to effectively learn a regression function $f_W$, parameterized by network weights $W$, that maps observational inputs $\mathbf{x}$ to estimated Euler angles
\begin{equation}
\label{eq:problem}
\hat{\boldsymbol{\theta}} = f_W(\mathbf{x}).
\end{equation}
Moreover, each joint $i$ operates within a finite interval
\begin{equation}
\label{eq:range}
    \theta_i \in I_i := [\theta_i^{\min},\,\theta_i^{\max}].
\end{equation}
Reflecting inherent anatomical or mechanical constraints, these range limits lie within a $2\pi$ period and are often substantially restricted.

To solve Equation~\ref{eq:problem}, existing approaches most commonly adopt standard MLP with the 6D representation~\cite{zhou2019continuity}. The predicted rotations are subsequently converted to Euler angles under a fixed convention, either to enforce angular range constraints during training~\cite{zhang2023body,zhang2024weakly} or for direct use in downstream tasks~\cite{xia2025reconstructing}. Such pipelines treat Euler angles primarily as post-hoc coordinates, disregarding their intrinsic angular structure and practical range constraints during learning. We instead argue for architectures with stronger functional adaptability to these coordinate properties, a capability fundamentally supported by the Kolmogorov-Arnold Network (KAN), as introduced below.

\subsection{Kolmogorov-Arnold Networks for Euler-Angle Regression}
\label{subsec:kan_euler}

We illustrate the architecture of KAN and highlight its key structural difference from standard MLP in Figure~\ref{fig:kan_euler}.

\begin{figure}[t]
\centering
\includegraphics[width=\columnwidth]{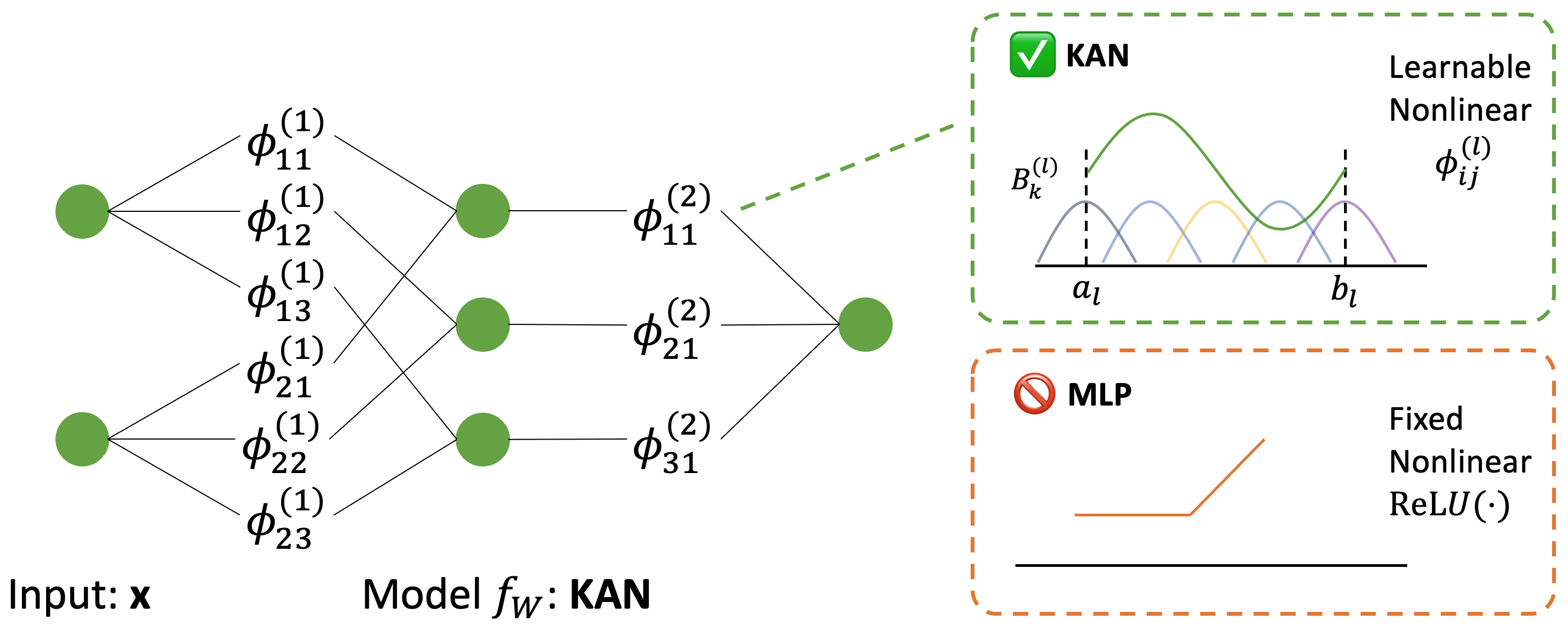}
\caption{
\textbf{Illustration of Kolmogorov-Arnold Networks (KAN) and their advantages over MLPs.} We show a two-input, one-hidden-layer example (green nodes), although practical models are typically deeper and wider. %
}
\label{fig:kan_euler}
\end{figure}

\noindent\textbf{Model Architecture.}
KAN retains a layered feedforward topology but moves the nonlinearity from nodes to edges by assigning a learnable univariate function to each directed connection.
For each layer and output channel, KAN computes
\begin{equation}
h^{(\ell+1)}_j
=
\sum_{i=1}^{d_\ell}
\phi^{(\ell)}_{ij}\!\left(h^{(\ell)}_i\right),
\label{eq:kan}
\end{equation}
where $\phi^{(\ell)}_{ij}:\mathbb{R}\rightarrow\mathbb{R}$ is a learnable scalar function associated with the edge from input channel $i$ in layer $\ell$ to output channel $j$ in layer $\ell+1$.
The final output is again $\hat{\boldsymbol{\theta}}=\mathbf{h}^{(L)}\in\mathbb{R}^{3}$.

In practice, each edge function $\phi^{(\ell)}_{ij}$ is parameterized on a bounded interval $[a_\ell,b_\ell]$ using a spline basis.
Concretely, we adopt a degree-$p$ B-spline expansion
\begin{equation}
\phi^{(\ell)}_{ij}(t)
=
\sum_{k=1}^{K_\ell}
c^{(\ell)}_{ij,k}\, B^{(\ell)}_{k}(t),
\quad t\in[a_\ell,b_\ell],
\label{eq:spline}
\end{equation}
where $\{B^{(\ell)}_{k}\}_{k=1}^{K_\ell}$ are fixed B-spline basis functions defined by a knot sequence on $[a_\ell,b_\ell]$, and $\{c^{(\ell)}_{ij,k}\}$ are learnable coefficients.
The integer $K_\ell$ determines the number of basis functions and thus controls the expressive capacity of edge nonlinearities in layer $\ell$.
This parameterization makes the nonlinearity explicitly learnable and channel-specific, in contrast to MLP where a fixed activation function, such as ReLU, is shared across all units.

\noindent\textbf{Alignment with the Trigonometric Structure of Euler Angles.}
The entries of a rotation matrix are polynomial combinations of $\sin(\cdot)$ and $\cos(\cdot)$ applied to individual Euler angles. For example, under a fixed $ZXY$ convention,
\begin{equation}
R_{11}= \cos\alpha\cos\gamma - \sin\alpha\sin\beta\sin\gamma, \quad
R_{12}= -\sin\alpha\cos\beta,
\end{equation}
and similarly for the remaining components.
Thus, each matrix entry is a product of smooth trigonometric functions of the individual angles.

Standard MLPs employ a shared fixed activation across all units, requiring greater depth and repeated compositions to approximate diverse nonlinear patterns. In contrast, KAN represents nonlinearities as learnable univariate functions defined on edges, allowing each channel to adapt its functional form to the specific angular behaviors. Since these nonlinearities are smooth functions defined on bounded intervals, they admit accurate approximation by spline bases. We provide a formal approximation error bound in Section~\ref{subsec:theoretical_analysis}. %

\noindent\textbf{Training Loss.}
During training, we supervise the predicted Euler angles 
$\hat{\boldsymbol{\theta}}$ with the corresponding ground-truth angles 
$\boldsymbol{\theta}$. Specifically, KAN is optimized using a standard 
regression objective
\begin{equation}
\mathcal{L}_{\text{Euler}}
=
\rho\!\left(\hat{\boldsymbol{\theta}}-\boldsymbol{\theta}\right),
\label{eq:euler_loss}
\end{equation}
where $\rho(\cdot)$ denotes a pointwise penalty function, e.g., the 
mean squared error. The intrinsic discontinuities and singularities of Euler 
angles typically hinder the training loss optimization. In the following, we introduce how explicit range constraints that naturally arise in practical systems mitigate these issues, thereby enabling stable training with the standard regression loss.

\subsection{Mitigating Euler-Angle Discontinuities and Singularities through Range Constraints}
\label{subsec:practical_guidelines}

Euler angles are periodic coordinates satisfying
$\theta_i \equiv \theta_i + 2\pi m$, where $m \in \mathbb{Z}$.
Their mapping to a rotation matrix is also rotation-order dependent,
which jointly leads to representation discontinuities and singular
configurations. However, when each admissible interval satisfies
$|I_i| < 2\pi$, periodic ambiguity is excluded within the
operational domain, eliminating the discontinuities induced by
equivalent $2\pi$ shifts of the angular coordinates.

We further choose the Euler rotation order to place the most constrained axis at the middle position to mitigate the gimbal-lock singularity. The choice is dictated by where the singularity lies: for the intrinsic $ZXY$ convention with angles $\boldsymbol{\theta}=(\theta_1,\theta_2,\theta_3)$ (the same triple denoted $[\alpha,\beta,\gamma]$ above), gimbal lock arises \emph{only} when the middle angle reaches $\theta_2 = \pm\tfrac{\pi}{2}$; the middle slot is therefore the sole singular one. Let $k \in \{1,2,3\}$ index the most constrained axis, whose admissible interval satisfies
\begin{equation}
    I_k \subset \left(-\tfrac{\pi}{2},\,\tfrac{\pi}{2}\right).
\end{equation}
We then adopt the Euler convention that assigns axis $k$ to the middle position. This choice leaves the represented orientation unchanged and only confines the singularity-prone middle angle to a singularity-free interval, satisfying $\theta_2 \in I_k \subset (-\tfrac{\pi}{2},\tfrac{\pi}{2})$. For example, in human hand articulation, finger twist (longitudinal rotation) is limited and often neglected; we assign this axis to the middle angle.

The resulting Euler chart is smooth and
locally well-conditioned over the compact domain
\begin{equation}
\label{eq:bound}
\Theta = I_1 \times I_2 \times I_3 \subset \mathbb{R}^3.    
\end{equation}
Consequently, Euler regression reduces to learning a mapping
\(
f:\mathcal{X} \to \Theta,
\)
where $\Theta$ is compact and free of singularities and periodic ambiguity. 

In practice, joints often exhibit only one or two rotational degrees of freedom~\cite{bensadoun2022neural,zhang2024weakly}. In these cases, the regression problem reduces to learning a low-dimensional mapping over bounded intervals, which is considerably simpler. We illustrate this for the hand and body models by providing the per-joint range specifications and axis assignments in Appendix~\ref{appx:singularity}.

\subsection{Theoretical Analysis of Euler-Angle Regression with KAN}
\label{subsec:theoretical_analysis}

This section explains why Euler-angle regression is well-suited to KAN in the range-constrained regime studied here. On the bounded chart, the target function decomposes into an additive component plus a residual interaction term, and KAN attains an approximation bound comprising a univariate spline term and a bounded interaction residual. Proofs are provided in Appendix~\ref{appx:theoretical_analysis}.

\noindent\textbf{Setup.}
We analyze the regression regime induced by the bounded, well-conditioned Euler chart $\Theta$ of Equation~\ref{eq:bound}, on which the Euler-to-rotation map $\Theta \to \mathrm{SO}(3)$ is smooth and everywhere nonsingular.

\begin{lemma}[Interaction control under bounded Euler variation]
\label{lem:euler_interaction_small}
Let $f:\mathcal{X}\to\Theta$ denote the Euler-angle regression map, where $\Theta$ is the bounded Euler chart in Equation~\ref{eq:bound}. Assume each component $f_i\in C^p(\mathcal{X})$ ($p\ge 3$) admits an ANOVA decomposition
\begin{equation}
f_i(x)
=
f_{i,0}
+
\sum_{j=1}^{d} f_{i,j}(x_j)
+
\sum_{|S|>1} f_{i,S}(x_S).
\end{equation}
For such $C^p$ functions on a bounded domain, the higher-order interaction mass is controlled by mixed partial derivatives:
\begin{equation}
\sum_{|S|>1}\|f_{i,S}\|_\infty
\le
C_{\mathcal X}
\max_{j\neq k}
\|\partial^2_{jk}f_i\|_\infty,
\label{eq:interaction_bound_general}
\end{equation}
where $C_{\mathcal X}$ depends only on the domain $\mathcal X$.

In the bounded Euler regime, the Euler-to-rotation map is a composition of elemental rotations. Locally, angular displacements can be expanded through the Baker--Campbell--Hausdorff formula:
\begin{equation}
\prod_i \exp(u_i G_i)
=
\exp\left(
\sum_i u_iG_i
+
\frac{1}{2}\sum_{j<k}u_j u_k [G_j,G_k]
+
O(\|u\|^3)
\right),
\label{eq:bch_local}
\end{equation}
where $u_i$ denotes the local angular displacement from a reference and $G_i$ are the corresponding rotation generators. The first-order term is additive across Euler coordinates, while the leading cross-coordinate coupling enters through the second-order commutator terms $\tfrac{1}{2}\sum_{j<k}u_j u_k [G_j,G_k]$, which are $O(\delta^2)$ in the chart half-width $\delta := \tfrac{1}{2}\max_k |I_k|$, with $I_k$ the admissible intervals of Equation~\ref{eq:bound}. This structure of the Euler-to-rotation map motivates modeling the regression target as additive-plus-residual; we denote its higher-order interaction residual by
\begin{equation}
\epsilon_{\mathrm{int}}
:=
\max_{i}\sum_{|S|>1}\|f_{i,S}\|_\infty .
\label{eq:epsilon_int}
\end{equation}
As the admissible Euler ranges tighten ($\delta$ small), the second-order coupling diminishes, so we expect $\epsilon_{\mathrm{int}}$ to decrease accordingly---a trend we confirm empirically in the range sweep of Appendix~\ref{sec:experiment_rotmat}.
\end{lemma}

Lemma~\ref{lem:euler_interaction_small} shows an additive-plus-residual structure, where the dominant term is additive and the interaction residual $\epsilon_{\mathrm{int}}$ stays small. This near-additivity is a bounded-domain effect, available here because Euler angles admit such bounded, singularity-free charts via the axis ordering of Section~\ref{subsec:practical_guidelines}. We now analyze the approximation behavior of KAN under this regime.

\begin{theorem}[Euler-aware approximation bound of KAN]
\label{thm:euler_kan_bound}
Assume Lemma~\ref{lem:euler_interaction_small}. Suppose further that $\mathcal{X}\subset \prod_{j=1}^{d}[a_j,b_j]$ and the univariate ANOVA components satisfy
\begin{equation}
f_{i,j}\in C^p([a_j,b_j]),
\qquad
\|f_{i,j}^{(p)}\|_\infty \le M,
\end{equation}
for all $i\in\{1,2,3\}$ and $j\in\{1,\dots,d\}$. Let $\mathcal{F}_{\mathrm{KAN}}(K,p)$ be the additive spline subclass of KAN with degree-$p$ splines and $K$ basis functions per coordinate. Then for each $i\in\{1,2,3\}$ there exists $\tilde f_i\in\mathcal{F}_{\mathrm{KAN}}(K,p)$ such that
\begin{equation}
\|f_i-\tilde f_i\|_\infty
\le
d\,C_p\,M\,K^{-p}
+
\epsilon_{\mathrm{int}},
\label{eq:kan_euler_bound}
\end{equation}
where $C_p$ depends only on $p$ and the interval lengths $\{b_j-a_j\}_{j=1}^{d}$.
\end{theorem}

In Equation~\ref{eq:kan_euler_bound}, the first term is the spline approximation error on the additive univariate components. The second term is the residual interaction error induced by remaining cross-coordinate coupling. Thus, when $\epsilon_{\mathrm{int}}$ is small, which occurs as the Euler ranges tighten, the target is well approximated by the additive spline form of the KAN architecture.

\begin{proposition}[Contextual comparison to generic MLP approximation guarantees]
\label{prop:mlp_comparison}
Let $f_i\in C^{p}(\mathcal{X})$, $p\ge 3$, where $\mathcal{X}\subset\mathbb{R}^d$ is bounded. Consider the class $\mathcal{F}_{\mathrm{MLP}}(W)$ of fully coupled ReLU MLPs with $W$ trainable parameters. A standard worst-case approximation guarantee for $C^p(\mathcal{X})$ functions yields
\begin{equation}
\inf_{g\in \mathcal{F}_{\mathrm{MLP}}(W)}
\|f_i-g\|_\infty
\le
C_{\mathrm{MLP}}W^{-p/d},
\label{eq:mlp_generic_rate}
\end{equation}
where $C_{\mathrm{MLP}}$ depends only on $\mathcal{X}$ and derivative bounds. In the bounded Euler regime of Theorem~\ref{thm:euler_kan_bound}, spline-KAN achieves
\begin{equation}
\|f_i-\tilde f_i\|_\infty
\le
d\,C_p\,M\,K^{-p}
+
\epsilon_{\mathrm{int}}.
\end{equation}
This comparison indicates that generic MLP guarantees apply to fully coupled $d$-variate functions, whereas the KAN bound exploits the additive-plus-residual structure arising in bounded Euler-angle regression. When $\epsilon_{\mathrm{int}}$ is small, KAN can therefore achieve a more parameter-efficient approximation of the dominant univariate components of the target.
\end{proposition}

This analysis explains why KAN+Euler is effective in the practical setting targeted by this work. Bounded Euler ranges remove periodic ambiguity, branch-consistent axis selection avoids singularities, and the remaining regression problem often exhibits a low-interaction structure over compact intervals. KAN matches this structure through learnable univariate spline functions on edges, while standard MLPs must learn the corresponding nonlinearities through repeated compositions of fixed activations.

\section{Experiments}
\label{sec:experiment}

To validate the proposed framework, we evaluate on a range of real-world rotation-regression tasks. We first characterize accuracy, efficiency, and convergence under controlled rotation ranges, using object rotation estimation in Section~\ref{sec:experiment_pointcloud}, and complement this analysis with Appendix~\ref{sec:experiment_rotmat}, which removes the downstream task by regressing rotations directly from rotation matrices. We then apply the framework to articulated systems: robotic-arm inverse kinematics in Section~\ref{sec:experiment_robot} and human-hand inverse kinematics in Section~\ref{sec:experiment_hand}, with a full-body articulation study in Appendix~\ref{sec:experiment_body} that extends these gains to a high-DoF model. In each section, we first present the experimental setup and then discuss the results. Additional implementation details are reported in Appendix~\ref{appx:implementation}.

\noindent\textbf{Evaluation Metrics.} Our primary metric is the Mean Angle Error (\textbf{MAE}), the average absolute difference between predicted and ground-truth Euler angles in degrees, with each angular difference wrapped to $(-\pi,\pi]$. For models trained with alternative parameterizations such as axis-angle (\textbf{AA}), quaternion, or the 6D representation, we follow common practice~\cite{xia2025reconstructing} and convert the predicted rotations to Euler angles under the same convention before computing the MAE. We further report the geodesic error (\textbf{GE}), the angle of the relative rotation between prediction and ground truth on $\mathrm{SO}(3)$, in degrees. For the inverse-kinematics tasks, we additionally report the forward-kinematics position error (\textbf{FKE}) in cm and the success rate under a $1$\,cm threshold (\textbf{SR@1cm}).

\subsection{Object Rotation}
\label{sec:experiment_pointcloud}

\noindent\textbf{Setup.}
We estimate an object's 3D orientation from its point cloud, using the \emph{chair} category of ModelNet10~\cite{wu2015modelnet} with 889 training and 100 test shapes. A shared PointNet-style encoder~\cite{qi2017pointnet} maps a canonical cloud and a rotated copy to global features, which are decoded to the relative rotation under the chosen representation and trained with a geodesic loss on $\mathrm{SO}(3)$. A division factor $\mathrm{div}$ scales each Euler-angle range by $1/\mathrm{div}$, bounding the gimbal-prone middle axis to $[-\pi/(2\,\mathrm{div}),\pi/(2\,\mathrm{div})]$. The accuracy comparison matches all models at ${\sim}1.44$M parameters, while the efficiency study varies model size; all reported results are averaged over 3 random seeds.

\begin{table}[tb]
\centering
\caption{\textbf{Object pose estimation.} ``Rep.'' refers to the employed rotation representation, and ``Params.'' to the number of model parameters.}
\label{tab:pointcloud}
\setlength{\tabcolsep}{2.5pt}
\begin{subtable}[t]{0.5\linewidth}
\centering
\caption{Accuracy comparison at $\mathrm{div}{=}1.3$}
\begin{tabular}{lccc}
\toprule
Model & Rep. & MAE $\downarrow$ & GE $\downarrow$ \\
\midrule
MLP & 6D & $4.02${\scriptsize$\pm.11$} & $7.01$ \\
MLP & Euler & $4.08${\scriptsize$\pm.06$} & $7.94$ \\
MLP & AA & $4.37${\scriptsize$\pm.11$} & $8.02$ \\
KAN & 6D & $3.03${\scriptsize$\pm.22$} & $5.33$ \\
KAN & AA & $2.57${\scriptsize$\pm.16$} & $4.78$ \\
\midrule
\textbf{KAN} & \textbf{Euler} & $\mathbf{2.40}${\scriptsize$\pm.20$} & $\mathbf{4.60}$ \\
\bottomrule
\end{tabular}
\end{subtable}%
\hfill
\begin{subtable}[t]{0.46\linewidth}
\centering
\caption{Efficiency comparison}
\begin{tabular}{lccc}
\toprule
Model & Params. & FLOPs & MAE $\downarrow$ \\
\midrule
MLP+6D & 0.42M & 0.56M & $7.18$ \\
MLP+6D & 0.73M & 1.18M & $5.56$ \\
MLP+6D & 1.46M & 2.63M & $4.02$ \\
MLP+6D & 3.29M & 6.30M & $3.67$ \\
\midrule
\textbf{KAN+Euler} & 0.45M & 0.51M & $\mathbf{2.98}$ \\
\bottomrule
\end{tabular}
\end{subtable}
\end{table}

\noindent\textbf{Accuracy.}
Table~\ref{tab:pointcloud}(a) reports accuracy across representations. KAN lowers GE under
every representation relative to its MLP counterpart: Euler from $7.94$ to $4.60$,
6D from $7.01$ to $5.33$, and AA from $8.02$ to $4.78$, so learnable univariate
activations benefit rotation regression regardless of the output space. The gain is largest
for the proposed KAN+Euler, which attains the best accuracy on both metrics ($2.40$ MAE,
$4.60$ GE), $40\%$ and $34\%$ below the widely adopted MLP+6D baseline ($4.02$, $7.01$),
and ahead of KAN+AA ($2.57$, $4.78$) and KAN+6D ($3.03$, $5.33$): under bounded
ranges, the native Euler coordinates expose the near-additive structure of
Section~\ref{subsec:theoretical_analysis} that KAN's splines are suited to fit, a structure
the overparameterized 6D output leaves implicit. Conversely, MLP+Euler ($7.94$ GE) is worse
than MLP+6D ($7.01$): with fixed activations, the angular nonlinearity of Euler coordinates
becomes a liability rather than an inductive bias. Accuracy is thus governed by the pairing
of representation and architecture, not by either choice alone.

\noindent\textbf{Efficiency.}
Table~\ref{tab:pointcloud}(b) scales the MLP+6D baseline by widening its hidden layers, reporting the
FLOPs of the representation head, since the shared PointNet encoder is common to all models. A compact
KAN+Euler, at $0.45$M parameters and $0.51$M head FLOPs, uses fewer FLOPs than even the smallest MLP+6D, yet
attains $2.98$ MAE, lower than MLP+6D at every budget, including the $7.3\times$ larger $3.29$M model
that reaches only $3.67$ at $6.30$M FLOPs. The gain therefore stems from the representation--architecture
match rather than model capacity, further demonstrating the efficiency of~our~framework.

\begin{figure}[tb]
\centering
\begin{subfigure}[b]{0.46\textwidth}\centering
\includegraphics[width=\textwidth]{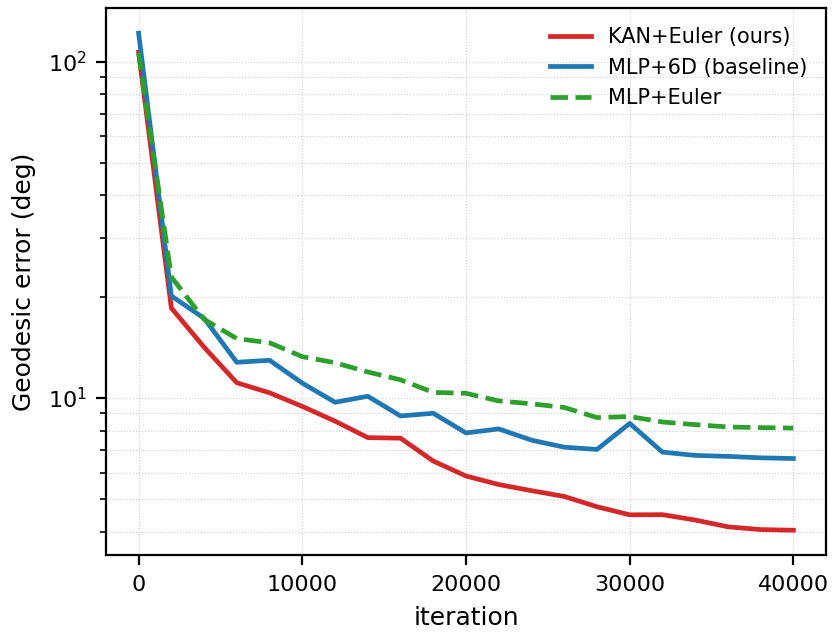}
\caption{Optimization convergence}\end{subfigure}
\hfill
\begin{subfigure}[b]{0.46\textwidth}\centering
\includegraphics[width=\textwidth]{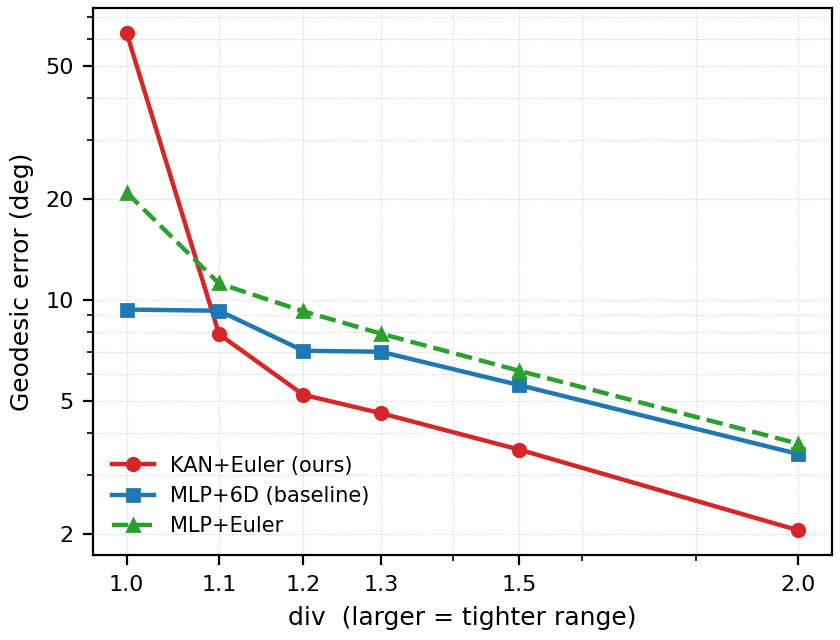}
\caption{Effect of range}\end{subfigure}
\caption{\textbf{Object pose estimation.}
\textbf{(a)} GE vs.\ training iteration for KAN+Euler, MLP+6D, and MLP+Euler at
$\mathrm{div}{=}1.3$. \textbf{(b)} GE vs.\ $\mathrm{div}$. Axes are log-scaled.}
\label{fig:pc_fig}
\end{figure}

\noindent\textbf{Convergence and Effect of Range.}
Figure~\ref{fig:pc_fig}(a) shows that KAN+Euler converges faster and to a lower error than the
matched MLP+6D, while MLP+Euler remains the slowest and least accurate throughout, consistent
with the accuracy analysis: a fixed-activation MLP cannot exploit the angular structure of the
Euler targets. Figure~\ref{fig:pc_fig}(b) identifies when the advantage holds: at the full range
($\mathrm{div}{=}1$) KAN+Euler performs poorly because sampled rotations reach the $\pm\pi/2$
gimbal set, where the Euler chart itself is ill-behaved; once the range excludes this
singularity, KAN+Euler crosses below MLP+6D near $\mathrm{div}{=}1.1$, and the margin widens to
$1.6\times$ by $\mathrm{div}{=}2$ as tighter ranges make the target increasingly near-additive.
The gains therefore track the bounded-range mechanism of
Section~\ref{subsec:theoretical_analysis} rather than any~task-specific~tuning.

\subsection{Robotic Arm Articulation}
\label{sec:experiment_robot}

\begin{table}[!b]
\centering
\caption{\textbf{Franka Panda IK results.}}
\label{tab:franka}
\setlength{\tabcolsep}{2.5pt}
\begin{subtable}[t]{0.44\linewidth}
\centering
\caption{Comparison across active DoFs}
\footnotesize
\begin{tabular}{clccc}
\toprule
DoF & Model & Rep. & FKE\,$\downarrow$ & SR@1cm\,$\uparrow$ \\
\midrule
\multirow{3}{*}{3}
 & MLP & Euler & 1.16 & 52.5 \\
 & MLP & 6D & 0.83 & 71.6 \\
 & \textbf{KAN} & \textbf{Euler} & \textbf{0.16} & \textbf{100.0} \\
\midrule
\multirow{3}{*}{4}
 & MLP & Euler & 1.94 & 20.2 \\
 & MLP & 6D & 1.15 & 50.0 \\
 & \textbf{KAN} & \textbf{Euler} & \textbf{0.44} & \textbf{96.8} \\
\midrule
\multirow{3}{*}{5}
 & MLP & Euler & 11.65 & 0.3 \\
 & MLP & 6D & 7.94 & 1.0 \\
 & \textbf{KAN} & \textbf{Euler} & \textbf{6.07} & \textbf{1.5} \\
\bottomrule
\end{tabular}
\end{subtable}
\hfill
\begin{subtable}[t]{0.54\linewidth}
\centering
\caption{Workspace constraints (3 DoFs)}
\footnotesize
\begin{tabular}{lccc}
\toprule
Method & Box & FKE\,$\downarrow$ & SR@1cm\,$\uparrow$ \\
\midrule
\multirow{3}{*}{MLP+6D}
 & Full & 0.788 & 74.0 \\
 & $0.7\times$ & 0.307 & 99.8 \\
 & $0.5\times$ & 0.239 & 99.9 \\
\midrule
\multirow{3}{*}{\textbf{KAN+Euler}}
 & Full & \textbf{0.168} & \textbf{100.0} \\
 & $0.7\times$ & \textbf{0.140} & \textbf{100.0} \\
 & $0.5\times$ & \textbf{0.082} & \textbf{100.0} \\
\bottomrule
\end{tabular}
\end{subtable}
\end{table}

\noindent\textbf{Setup.}
We evaluate on the Franka Panda robotic arm, a 7-revolute-joint manipulator widely used in manipulation research. Each revolute joint angle $q_i$ is a bounded 1D Euler rotation about its fixed joint axis, so the regression target (the vector of joint angles) is itself a product of range-constrained Euler coordinates, and the end-effector rotation (given to the network in the compared representation) is their forward-kinematic composition. Given the end-effector pose (3D position and rotation matrix), networks are trained to predict the active joint angles on 1M sampled~pose--configuration~pairs. To evaluate our approach across task difficulty, we vary the number of active degrees of freedom (DoFs) from 3 to 5; higher DoFs are not evaluated because the inverse problem becomes increasingly ambiguous under a fixed end-effector observation~\cite{bensadoun2022neural}. To ensure a fair comparison, all models are matched to ${\sim}$100K parameters, trained under identical settings, and averaged over 3 random seeds.

\noindent\textbf{Comparison Across Active DoFs.} As shown in Table~\ref{tab:franka}(a), the proposed KAN+Euler attains the lowest FKE at every DoF count. With 3 active DoFs it achieves 0.16\,cm at 100.0\% SR@1cm, a $5.2\times$ lower error than MLP+6D (0.83\,cm, 71.6\%); with 4 DoFs it reaches 0.44\,cm at 96.8\%, against 1.15\,cm at 50.0\%. At 5 DoFs the inverse problem is near-degenerate~\cite{bensadoun2022neural}: KAN+Euler remains the most accurate (6.07\,cm vs.\ 7.94\,cm for MLP+6D and 11.65\,cm for MLP+Euler), yet SR@1cm is low for all methods (at most 1.5\%) at this threshold. The relative advantage thus grows monotonically as fewer joints are active ($1.3\times$ at 5 DoFs, $2.6\times$ at 4, $5.2\times$ at 3), matching our analysis: bounded joint ranges keep Euler regression away from singularities, and KAN's additive spline form best fits the resulting near-additive structure of the joint angles.

\noindent\textbf{Advantage Under Workspace Constraints.} Real deployments rarely use the full reachable workspace, which further tightens the effective joint ranges. We therefore test robustness to workspace size, re-running the 3-DoF setting at the default box and at boxes shrunk to $0.7\times$ and $0.5\times$ its extent, at the same 1M-sample budget. As Table~\ref{tab:franka}(b) shows, KAN+Euler is the most accurate at every workspace, with $100\%$ SR@1cm and a consistent $2.2$--$4.7\times$ lower FKE than MLP+6D, whereas MLP+6D becomes reliable only as the box tightens, its SR@1cm falling to $74\%$ at the default box. A tighter workspace, as in real robotic deployments, shrinks the effective joint ranges into precisely the bounded regime that our method is designed to exploit.

\subsection{Biomechanical Hand Articulation}
\label{sec:experiment_hand}

\begin{table}[!b]
\centering
\caption{\textbf{Biomechanical inverse kinematics on the human hand.}
``Constr.'' indicates whether the practical guidelines introduced in Section~\ref{subsec:practical_guidelines} are applied to mitigate singularities and discontinuities. ``Data'' denotes the percentage of the training set used in the data-efficiency study.}
\label{tab:hand_articulation}
\setlength{\tabcolsep}{4pt}
\begin{subtable}[t]{0.48\linewidth}
\centering
\caption{Accuracy comparison}
\begin{tabular}{lcccc}
\toprule
Model & Rep. & Constr. & MAE $\downarrow$ & GE $\downarrow$ \\
\midrule
MLP & 6D & - & $2.72${\scriptsize$\pm.03$} & $4.90$ \\
MLP & Euler & \xmark & $3.72${\scriptsize$\pm.02$} & $6.28$ \\
MLP & Euler & \cmark & $3.15${\scriptsize$\pm.02$} & $5.64$ \\
MLP & AA & - & $2.68${\scriptsize$\pm.05$} & $4.89$ \\
\midrule
KAN & 6D & - & $2.86${\scriptsize$\pm.07$} & $5.22$ \\
KAN & Euler & \xmark & $3.10${\scriptsize$\pm.03$} & $5.14$ \\
KAN & AA & - & $2.69${\scriptsize$\pm.03$} & $4.90$ \\
\midrule
\textbf{KAN} & \textbf{Euler} & \cmark & $\mathbf{2.53}${\scriptsize$\pm.03$} & $\mathbf{4.61}$ \\
\bottomrule
\end{tabular}
\end{subtable}
\hfill
\begin{subtable}[t]{0.48\linewidth}
\centering
\caption{Training data-efficiency analysis}
\begin{tabular}{lcc}
\toprule
Method & Data & MAE $\downarrow$ \\
\midrule
MLP+6D & 10 & $2.94${\scriptsize$\pm.05$}  \\
MLP+6D & 50 & $2.81${\scriptsize$\pm.06$}  \\
MLP+6D & 100 & $2.72${\scriptsize$\pm.03$}  \\
\midrule
KAN+Euler & 10 & $2.73${\scriptsize$\pm.07$}  \\
KAN+Euler & 50 & $2.58${\scriptsize$\pm.03$}  \\
\textbf{KAN+Euler} & 100 & $\mathbf{2.53}${\scriptsize$\pm.03$}  \\
\bottomrule
\end{tabular}
\end{subtable}
\end{table}

\noindent\textbf{Setup.}
Following established protocols in biomechanical inverse kinematics analysis~\cite{xia2025reconstructing}, we train neural networks to estimate joint rotation angles from sparse observations, namely 3D hand joint positions. Specifically, we employ the FreiHand dataset~\cite{zimmermann2019freihand}, which features a diverse collection of hand poses from multiple subjects. The dataset is annotated using the widely adopted MANO hand model~\cite{romero2017embodied}, which captures 21 3D hand joint positions and 15 articulated joints (45 rotation angles). It contains 33K training and~4K~testing~samples.

To ensure a fair comparison, we evaluate MLP and KAN under matched capacity (approximately 34K parameters), averaging over 3 seeds. We compare rotation representations and isolate the effect of the range-constraint guidelines of Section~\ref{subsec:practical_guidelines}, which the Constraint column of Table~\ref{tab:hand_articulation}(a)~\mbox{toggles}.

\noindent\textbf{Effect of Range Constraints.} Table~\ref{tab:hand_articulation}(a) shows that regressing Euler angles without mitigating discontinuities and singularities degrades accuracy: unconstrained MLP+Euler reaches an MAE of 3.72, well behind MLP+6D at 2.72. Applying our proposed range constraints and constraint-aware axis ordering from Section~\ref{subsec:practical_guidelines} lowers this to 3.15, recovering most of the gap to MLP+6D with no change to the architecture. The weakness of unconstrained Euler regression therefore comes from these discontinuities and singularities, which our range constraints avoid by keeping the angles within their bounded operating ranges, a standalone, architecture-agnostic gain from which even a~plain~\mbox{MLP}~\mbox{benefits}.

\noindent\textbf{Effectiveness of KAN for Euler Regression.} KAN is markedly more effective than MLP at regressing Euler angles. Even without constraints, KAN+Euler reaches 3.10, improving on MLP+Euler by 16.7\% because its learnable univariate activations fit the bounded, near-additive structure of the joint-angle targets that a fixed-activation MLP cannot exploit. Combined with the range constraints, KAN+Euler attains the best accuracy of 2.53, ahead of KAN+6D (2.86) by 11.5\% and of the strongest MLP baseline, MLP+AA (2.68), by 5.6\%. The two contributions are complementary: the constraints make the Euler target well-behaved, and the KAN architecture exploits the resulting structure.

\noindent\textbf{Training Data Efficiency.} Table~\ref{tab:hand_articulation}(b) compares KAN+Euler against MLP+6D as the training set shrinks. KAN+Euler is more accurate at every budget, and at only 10\% of the data it attains 2.73, better than MLP+6D at the same fraction (2.94) and on par with MLP+6D trained on the full set (2.72). The representation--architecture match thus improves sample efficiency: because the additive spline hypothesis space already matches the target structure, fewer examples are needed to fit it, a valuable property in biomechanical settings where labeled data are typically scarce.

\section{Conclusion}
We revisited Euler-angle regression and showed that the common preference for redundant, continuous parameterizations is unnecessary in the bounded-range regime of real articulated systems. The difficulty of regressing Euler angles arises not from the coordinates themselves but from a mismatch between fixed-activation networks and their angular structure. We resolved this mismatch on three fronts: introducing Kolmogorov-Arnold Networks (KAN) as a functionally adaptive backbone that exploits the near-additive structure of bounded Euler targets; leveraging practical range constraints with a constraint-aware axis ordering to remove discontinuities and singularities; and providing a theoretical analysis characterizing when the resulting target favors KAN's additive spline form. Across controlled rotation regression, object pose estimation, and robotic and human inverse kinematics, the framework consistently improves accuracy, convergence, and efficiency over the widely adopted MLP+6D baseline, while operating directly in the interpretable native coordinates of each system. Our analysis assumes range constraints on individual joints; extending it to richer structure, such as inter-joint synergies and coupled kinematic constraints, is a promising direction for integrating structural priors into~\mbox{rotation}~regression.

\bibliography{main}
\bibliographystyle{tmlr}

\clearpage
\appendix
\begin{center}
{\large\bfseries Revisiting Euler-Angle Regression with Kolmogorov-Arnold Networks}\\[0.4em]
{\large Supplementary Material}
\end{center}
\vspace{0.8em}

\noindent In this supplementary material, we first provide additional analysis and specifications underlying our method:
\begin{itemize}\setlength{\itemsep}{2pt}
  \item \textbf{Appendix~\ref{appx:singularity}}: Per-Joint Range Specifications and Singularity Handling.
  \item \textbf{Appendix~\ref{appx:theoretical_analysis}}: Theoretical Analysis of Euler-Angle Regression with KAN.
\end{itemize}
\noindent We then report additional experiments:
\begin{itemize}\setlength{\itemsep}{2pt}
  \item \textbf{Appendix~\ref{sec:experiment_rotmat}}: Controlled Rotation Regression.
  \item \textbf{Appendix~\ref{sec:experiment_body}}: Biomechanical Body Articulation.
\end{itemize}
\noindent Finally, we detail the experimental setup:
\begin{itemize}\setlength{\itemsep}{2pt}
  \item \textbf{Appendix~\ref{appx:implementation}}: Additional Implementation Details.
\end{itemize}
\vspace{0.5em}

\suppressfloats[t]

\input{appendix}
\end{document}

%% file: math_commands.tex
\usepackage{amsmath,amsfonts,bm}

\def\eqref#1{equation~\ref{#1}}

\def\1{\bm{1}}

\DeclareMathAlphabet{\mathsfit}{\encodingdefault}{\sfdefault}{m}{sl}
\SetMathAlphabet{\mathsfit}{bold}{\encodingdefault}{\sfdefault}{bx}{n}



%% file: appendix.tex
\section{Per-Joint Range Specifications and Singularity Handling}
\label{appx:singularity}

This appendix specifies the admissible per-joint rotation ranges of the human
hand and body, drawn from biomechanical literature, and details how the
axis-assignment strategy of Section~\ref{subsec:practical_guidelines} renders the
Euler chart free of gimbal lock for every joint whose most-constrained axis lies
strictly within $(-\tfrac{\pi}{2},\tfrac{\pi}{2})$---every joint we model except the
shoulder, whose axial-rotation axis reaches the $\pm\tfrac{\pi}{2}$ boundary (discussed below). For each
joint we report its rotational degrees of freedom (DoF) together with the
admissible interval of each axis, highlighting the axis assigned to the
gimbal-prone middle position.

\begin{table}[t]
\centering
\footnotesize
\setlength{\tabcolsep}{3.5pt}
\renewcommand{\arraystretch}{1.05}
\caption{\textbf{Per-joint DoF and admissible Euler-axis ranges} (degrees), compiled from
biomechanical literature~\cite{zhang2023body,zhang2024weakly,delp2007opensim}
for the hand and body. The \textbf{middle} axis is the most-constrained axis of each
joint; it lies within $(-90^\circ,90^\circ)$ for all joints except the shoulder, whose
axial-rotation axis reaches the $\pm90^\circ$ boundary. ``--'' marks axes that are inactive for the given joint.}
\label{tab:joint_ranges}
\begin{tabular}{@{}lccc@{}}
\toprule
Joint (DoF) & First axis & \textbf{Middle axis} & Third axis \\
\midrule
\multicolumn{4}{@{}l}{\textit{Hand}}\\
MCP (2) & $[-45,90]$ & $\mathbf{[-10,10]}$ & $[-20,20]$ \\
PIP (1) & $[-5,110]$ & --                  & -- \\
DIP (1) & $[-10,80]$ & --                  & -- \\
\midrule
\multicolumn{4}{@{}l}{\textit{Body}}\\
Knee (1)     & $[0,140]$   & --                  & --        \\
Elbow (1)    & $[0,146]$   & --                  & --        \\
Forearm (1)  & $[-80,90]$  & --                  & --        \\
Ankle (1)    & $[-20,50]$  & --                  & --        \\
Wrist (2)    & $[-70,80]$  & $\mathbf{[-25,35]}$ & --        \\
Hip (3)      & $[-30,120]$ & $\mathbf{[-45,30]}$ & $[-45,45]$ \\
Spine$^\ast$ (3) & $[-35,70]$ & $\mathbf{[-20,20]}$ & $[-30,30]$ \\
Scapula (3)  & $[-30,30]$ & $\mathbf{[-10,40]}$ & $[0,60]$ \\
Shoulder (3) & $[0,180]$    & $\mathbf{[-70,90]}$ & $[-60,180]$ \\
\bottomrule
\end{tabular}
\\[3pt]
{\footnotesize $^\ast$Lumbar segment (middle axis = axial rotation); the thoracic
and cervical segments are analogous.}
\end{table}

\subsection{Hand}
The MANO hand model~\cite{romero2017embodied} represents the hand as a kinematic
tree of $16$ joints rooted at the wrist: five fingers, each articulated by three
joints. The four fingers (index, middle, ring, and little) each comprise a
metacarpophalangeal (MCP), a proximal interphalangeal (PIP), and a distal
interphalangeal (DIP) joint, while the thumb comprises the carpometacarpal,
metacarpophalangeal, and interphalangeal joints. The DIP and PIP joints are
predominantly single-DoF flexion joints, whereas the MCP joints additionally
admit abduction; in all cases the longitudinal twist is tightly constrained about
zero~\cite{zhang2024weakly}. We adopt the intrinsic $ZXY$ convention, which places
this twist axis at the middle position. As Table~\ref{tab:joint_ranges} shows, the
middle-axis interval remains well inside $(-\tfrac{\pi}{2},\tfrac{\pi}{2})$ for
every finger-joint family, while the larger flexion ranges occupy the outer first
and third positions where no discontinuity or singularity can arise.

\subsection{Body}
The SKEL model~\cite{keller2024skel} represents the body as a kinematic tree of $24$ joints rooted at the pelvis. We focus on the major articulated joints: the
legs (hip, knee, ankle), the spine (lumbar, thoracic, and cervical), and the arms
(scapula, shoulder, elbow, forearm, and wrist). These joints span
three DoF classes: single-DoF hinges (knee, ankle, elbow, and forearm
pronation/supination), the two-DoF wrist, and three-DoF spinal and ball joints (hip, lumbar, thoracic, cervical, scapula, and shoulder). Under the anatomical
joint limits established in the biomechanical literature~\cite{zhang2023body,
delp2007opensim}, every multi-DoF joint has a most-constrained axis that we place at the middle position: ab/adduction for the hip, axial rotation for the spinal segments, and elevation for the scapula, each lying within $(-\tfrac{\pi}{2},\tfrac{\pi}{2})$. The shoulder is the sole exception: its most-constrained axis, axial rotation, reaches $\pm\tfrac{\pi}{2}$, making it the most singularity-prone joint we model. The single-DoF hinges are
singularity-free under any axis assignment.

Notably, for both the hand and the body these anatomical limits are far more
restrictive than the crossover range at which KAN+Euler overtakes the 6D
representation in our controlled study (near $\mathrm{div}{=}1.3$; at $\mathrm{div}{=}2$, i.e.\ a bounded axis within
$(-\tfrac{\pi}{4},\tfrac{\pi}{4})$ and the remaining axes within $(-\tfrac{\pi}{2},\tfrac{\pi}{2})$, the advantage already reaches $3.9\times$): most middle
axes are confined to a small fraction of $(-\tfrac{\pi}{2},\tfrac{\pi}{2})$, and the flexion ranges of the hinge joints (elbow $146^\circ$, knee $140^\circ$) stay below a half-turn. Real articulated systems therefore operate firmly inside the regime where the proposed parameterization is most advantageous.

\section{Theoretical Analysis of Euler-Angle Regression with KAN}
\label{appx:theoretical_analysis}

\noindent\textbf{Proof of Lemma~\ref{lem:euler_interaction_small}.}
Since $f_i\in C^p(\mathcal{X})$ ($p\ge3$, hence $f_i\in C^2$) and $\mathcal{X}\subset\mathbb{R}^d$ is bounded, the functional ANOVA decomposition exists and can be written as
\[
f_i(x)
=
f_{i,0}
+
\sum_{j=1}^d f_{i,j}(x_j)
+
\sum_{|S|>1} f_{i,S}(x_S),
\]
where the terms $\{f_{i,S}\}$ are orthogonal interaction components defined with respect to the product measure on $\mathcal{X}$ (see, e.g.,~\cite{rabitz1999general,hooker2007generalized}).
For twice continuously differentiable functions on a bounded domain, the magnitude of higher-order interaction components can be controlled by mixed second derivatives. In particular, standard ANOVA--Sobolev bounds imply
\begin{equation}
\sum_{|S|>1} \|f_{i,S}\|_\infty
\le
C_{\mathcal{X}}
\max_{j\neq k}
\|\partial_{jk}^2 f_i\|_\infty,
\end{equation}
where $C_{\mathcal{X}}$ depends only on the geometry of the domain $\mathcal{X}$ (see, e.g.,~\cite{rabitz1999general,hooker2007generalized}).

We then relate this interaction view to the bounded Euler chart considered in the paper. In a fixed Euler convention, the Euler-to-rotation map is a product of elemental rotations. For a local displacement $\mathbf{u}$ from a reference Euler angle, the Baker--Campbell--Hausdorff expansion gives
\begin{equation}
\prod_i \exp(u_iG_i)
=
\exp\left(
\sum_i u_iG_i
+
\frac{1}{2}\sum_{j<k}u_ju_k[G_j,G_k]
+
O(\|\mathbf{u}\|^3)
\right).
\end{equation}
The first-order term is additive across the Euler coordinates, while cross-coordinate coupling enters only through the second-order commutator terms, which are $O(\|\mathbf{u}\|^2)$ in the chart half-width. This near-additive structure of the Euler-to-rotation map motivates modeling the regression target as an additive component plus a residual interaction term, whose interaction mass is controlled by the bound above. We denote this residual interaction mass by
\begin{equation}
\epsilon_{\mathrm{int}}
:=
\max_{i}\sum_{|S|>1}\|f_{i,S}\|_\infty .
\end{equation}
This completes the interaction-control characterization used in the main text.

\noindent\textbf{Proof of Theorem~\ref{thm:euler_kan_bound}.}
Decompose
\[
f_i(x)=G_i(x)+r_i(x),
\qquad
G_i(x):=\sum_{j=1}^d f_{i,j}(x_j),
\qquad
r_i(x):=\sum_{|S|>1} f_{i,S}(x_S).
\]
By Lemma~\ref{lem:euler_interaction_small},
\[
\|r_i\|_\infty
=
\Big\|\sum_{|S|>1} f_{i,S}\Big\|_\infty
\le
\sum_{|S|>1}\|f_{i,S}\|_\infty
\le
\epsilon_{\mathrm{int}} .
\]
For each coordinate $j$, classical spline approximation theory implies that for any
$f_{i,j}\in C^p([a_j,b_j])$ there exists a degree-$p$ spline $s_{i,j}$ with $K$ basis
functions such that
\[
\|f_{i,j}-s_{i,j}\|_\infty
\le
C_p\,M\,K^{-p},
\]
where $C_p$ depends only on $p$ and the interval length $b_j-a_j$
(see, e.g.,~\cite{deboor2001practical}).
Define the additive spline approximation
\[
\tilde f_i(x)
:=
\sum_{j=1}^d s_{i,j}(x_j),
\]
which belongs to $\mathcal{F}_{\mathrm{KAN}}(K,p)$ by construction.
Then
\[
\|G_i-\tilde f_i\|_\infty
=
\Big\|\sum_{j=1}^d (f_{i,j}-s_{i,j})\Big\|_\infty
\le
\sum_{j=1}^d \|f_{i,j}-s_{i,j}\|_\infty
\le
d\,C_p\,M\,K^{-p}.
\]
Finally,
\[
\|f_i-\tilde f_i\|_\infty
=
\|G_i+r_i-\tilde f_i\|_\infty
\le
\|G_i-\tilde f_i\|_\infty+\|r_i\|_\infty
\le
d\,C_p\,M\,K^{-p}
+
\epsilon_{\mathrm{int}},
\]
which proves the bound in Equation~\ref{eq:kan_euler_bound}.

\section{Controlled Rotation Regression}
\label{sec:experiment_rotmat}

\noindent\textbf{Setup.}
Following \citet{zhou2019continuity}, we study rotation regression in a controlled setting while isolating complexity from downstream tasks. Given a $3\times3$ rotation matrix $R\in\mathrm{SO}(3)$ as input, we train neural networks to recover the corresponding ZXY Euler angles. Following the axis-assignment guideline of Section~\ref{subsec:practical_guidelines}, the gimbal-prone middle axis is sampled from $[-\pi/2,\pi/2]$ and the two outer axes from $[-\pi,\pi]$; at $\mathrm{div}=1$ the middle axis spans the full $[-\pi/2,\pi/2]$ and thus touches the $\pm\pi/2$ singularity, while larger $\mathrm{div}$ confines it strictly inside. We generate 500K samples for training and 50K independent samples for testing.

We evaluate KAN and MLP under two rotation parameterizations: Euler angles and the 6D representation. Since MLP with the 6D representation is widely adopted and achieves strong empirical performance for rotation regression, we treat MLP+6D as our primary baseline. We first compare accuracy across representations together with parameter and computation efficiency. Then, we analyze optimization stability and the effect of restricting Euler-angle ranges. A division factor $\mathrm{div}$ that scales each Euler-angle range by $1/\mathrm{div}$ (e.g., $\mathrm{div}=2$ corresponds to half the range) is used to examine how bounded Euler ranges affect results. Unless otherwise stated, results use $\mathrm{div}=2$ and are averaged over 3 seeds.

\noindent\textbf{Accuracy.} We report the accuracy comparison across representations in
Table~\ref{tab:acc_eff}(a). %
Regressing Euler angles with an MLP yields a large MAE of 0.546, reflecting the discontinuities and non-smooth regions of the Euler parameterization. KAN+Euler reduces this to 0.063, a $4.4\times$ improvement over the MLP+6D baseline (0.277). The gain from KAN holds across representations: it also lowers the error to 0.072 for 6D and 0.033 for AA. AA is the lowest in this synthetic setting, which is expected: the network is given the rotation matrix directly, and AA relates to it almost linearly through the matrix logarithm, making it the easiest target to regress. This edge is specific to direct rotation-matrix input and does not transfer to realistic observations. On tasks where the input is task observations rather than the rotation matrix itself, KAN+Euler is the strongest representation: it matches or exceeds AA on the point-cloud (MAE $2.40$ vs.\ $2.57$) and hand ($2.53$ vs.\ $2.69$) benchmarks, and is the best representation on the body task, as shown in Sections~\ref{sec:experiment_pointcloud} and~\ref{sec:experiment_hand} and Appendix~\ref{sec:experiment_body}. We therefore adopt Euler angles as the primary representation: they are the native, interpretable coordinates of the target articulated systems and match or exceed every evaluated alternative on the realistic benchmarks.

\begin{table}[tb]
\centering
\caption{\textbf{Rotation regression accuracy and efficiency.} ``Rep.'' refers to the employed rotation representation, and ``Params.'' to the number of model parameters. Bold marks our proposed KAN+Euler; on this synthetic matrix-input task KAN+AA is numerically lowest (see text).}
\label{tab:acc_eff}
\setlength{\tabcolsep}{4pt}
\begin{subtable}[t]{0.47\linewidth}
\centering
\caption{Accuracy comparison}
\begin{tabular}{lccc}
\toprule
Model & Rep. & MAE $\downarrow$ & GE $\downarrow$ \\
\midrule
MLP & 6D & $0.277${\scriptsize$\pm.008$} & $0.170$ \\
MLP & Euler & $0.546${\scriptsize$\pm.038$} & $0.356$ \\
MLP & AA & $0.363${\scriptsize$\pm.004$} & $0.224$ \\
KAN & 6D & $0.072${\scriptsize$\pm.008$} & $0.043$ \\
KAN & AA & $0.033${\scriptsize$\pm.002$} & $0.021$ \\
\midrule
\textbf{KAN} & \textbf{Euler} & $\mathbf{0.063}${\scriptsize$\pm.010$} & $\mathbf{0.041}$ \\
\bottomrule
\end{tabular}
\end{subtable}%
\hfill
\begin{subtable}[t]{0.5\linewidth}
\centering
\caption{Efficiency comparison}
\begin{tabular}{lcccc}
\toprule
Model & Rep. & Params. & FLOPs & MAE $\downarrow$ \\
\midrule
MLP & 6D & 5,331 & 10,368 & 0.277 \\
MLP & 6D & 21,806 & 43,000 & 0.147 \\
MLP & 6D & 100,766 & 200,200 & 0.081 \\
MLP & 6D & 185,406 & 369,000 & 0.066 \\
\midrule
\textbf{KAN} & \textbf{Euler} & 5,859 & 5,824 & \textbf{0.063} \\
\bottomrule
\end{tabular}
\end{subtable}
\end{table}

\noindent\textbf{Efficiency.} We compare parameter count and computational cost in Table~\ref{tab:acc_eff}(b). KAN+Euler reaches an MAE of 0.063 with only 5,859 parameters and 5,824 FLOPs, whereas MLP+6D at a comparable size yields 0.277; even scaled to 185K parameters and 369K FLOPs, a $32\times$ larger budget, MLP+6D remains at 0.066. The gain is therefore not a capacity effect but stems from the representation--architecture match, consistent with Section~\ref{sec:experiment_pointcloud}. We also report inference cost: a KAN forward pass is slower than the matched MLP+6D ($905\,\mu$s vs.\ $36\,\mu$s), a gap that reflects unfused B-spline kernels in the implementation rather than the parameter count.

\noindent\textbf{Optimization Stability.}
We next verify the convergence advantage of our framework. Training both models under a common angle-space MSE loss, Figure~\ref{fig:loss_conv} shows that KAN+Euler reaches a lower training loss and a lower Euler-angle error and decreases smoothly throughout, whereas MLP+6D plateaus at a higher level. This is exactly what the bounded-range analysis of Section~\ref{subsec:theoretical_analysis} predicts: the singularity-free, near-additive Euler target yields a well-conditioned landscape that KAN's additive spline form fits smoothly. These results support the improved optimization stability of our proposed framework.

\begin{figure}[tb]
\centering
\includegraphics[width=0.99\textwidth]{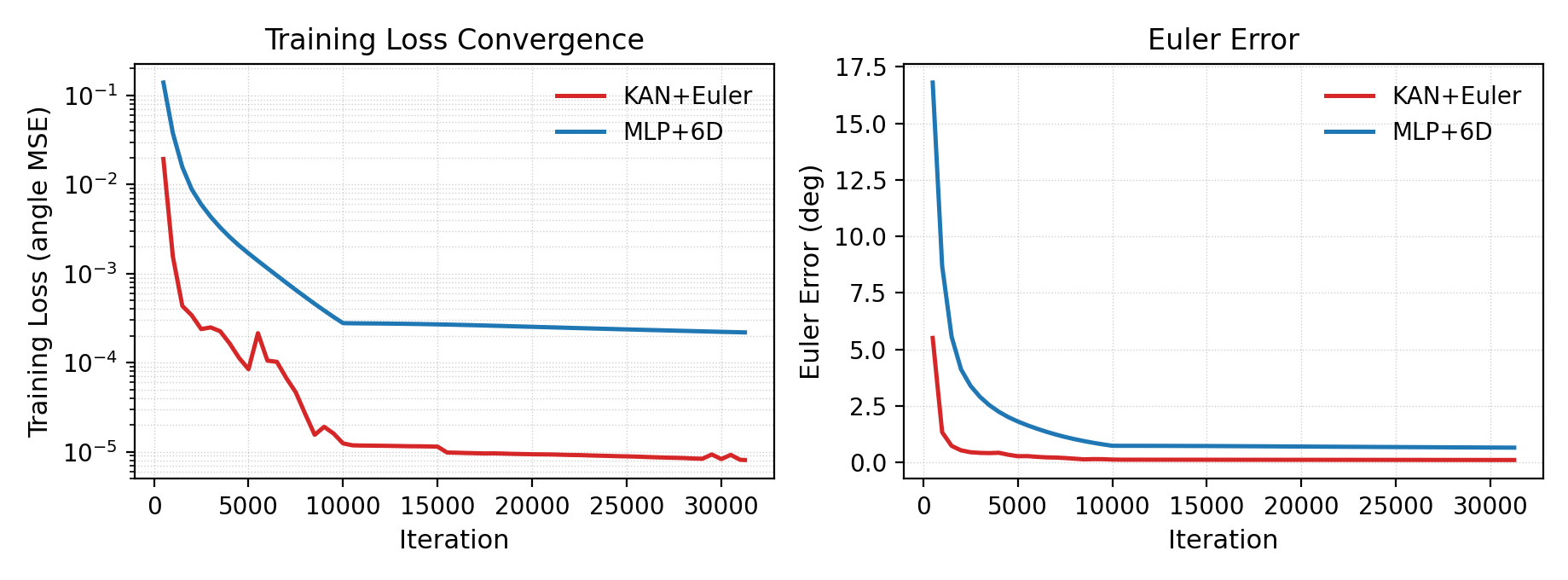}
\caption{\textbf{Optimization on controlled rotation regression.} Both models are trained with the same angle-space MSE loss. Left: training loss (log scale); right: Euler-angle error (MAE, degrees).}
\label{fig:loss_conv}
\end{figure}

\noindent\textbf{Effect of Angle Ranges.}  Finally, we investigate how the Euler-angle range affects regression difficulty by varying $\mathrm{div}$ in Figure~\ref{fig:angle_range}.
\begin{wrapfigure}{r}{0.45\textwidth}
\centering
\includegraphics[width=\linewidth]{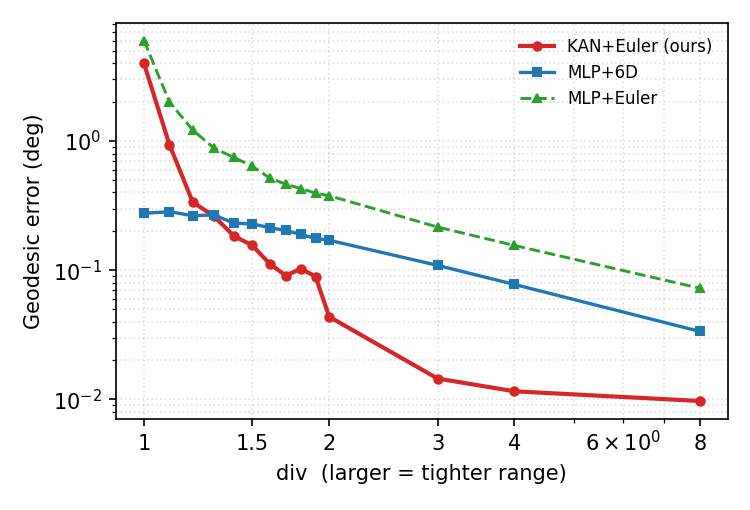}
\caption{\textbf{Effect of Euler-angle range.} GE vs.\ $\mathrm{div}$; axes are log-scaled.}
\label{fig:angle_range}
\end{wrapfigure}
When the full range is used ($\mathrm{div}=1$), both Euler variants perform poorly: KAN+Euler (3.99 GE) and MLP+Euler (6.01) trail MLP+6D (0.28), because samples near the range boundaries still suffer from discontinuities and singularities. As the range tightens, KAN+Euler improves sharply, crossing below MLP+6D near $\mathrm{div}{=}1.3$ and reaching $3.9\times$ lower error by $\mathrm{div}{=}2$ (0.044 vs.\ 0.171 in this single-seed sweep, consistent with Table~\ref{tab:acc_eff}(a)). The advantage peaks at $6.7\times$ ($\mathrm{div}{=}4$) and compresses to $3.5\times$ at $\mathrm{div}{=}8$ as both errors approach the evaluation floor. MLP+Euler improves monotonically but remains the weakest, showing that both the bounded range and the KAN architecture are needed. This supports the analysis in Section~\ref{subsec:theoretical_analysis}: restricting the Euler ranges reduces discontinuities and improves local smoothness, so KAN's additive structure better fits the resulting well-behaved bounded regime.

\section{Biomechanical Body Articulation}
\label{sec:experiment_body}

\noindent\textbf{Setup.}
We finally evaluate on the full-body SKEL model~\cite{keller2024skel}, the most demanding
articulated system we consider: 23 articulated joints along a deep kinematic chain, including
$3$-DoF ball joints (shoulder, hip, lumbar) whose broad coupled ranges sit near the boundary
of the bounded-Euler regime. From AMASS-derived poses~\cite{mahmood2019amass} we regress each
joint's rotation from $3$D joint positions. We operate in the model's \emph{native anatomical
coordinates}: the bounded joint parameters defined by the biomechanical model, with their
per-axis range limits and axis assignments detailed in Appendix~\ref{appx:singularity}, rather than a
generic Euler decomposition that would place an unconstrained axis at the gimbal-prone middle
position for ball joints. Because the chain is deep, we predict each joint from its
\emph{local} kinematic neighborhood (incoming bone, parent context, and child-bone directions)
with a shared per-joint network; this decouples the chain so that each output depends on a
local, near-additive feature set. KAN+Euler and MLP+6D are matched at ${\sim}96$K parameters
and trained over 3 seeds.

\begin{table}[tb]
\centering
\setlength{\tabcolsep}{8pt}
\caption{\textbf{Body articulation on the SKEL model.}}
\label{tab:body}
\begin{tabular}{lcc}
\toprule
Method & MAE $\downarrow$ & GE $\downarrow$ \\
\midrule
MLP+6D & $3.06${\scriptsize$\pm.03$} & $4.40${\scriptsize$\pm.03$} \\
KAN+6D & $3.11${\scriptsize$\pm.06$} & $4.42${\scriptsize$\pm.07$} \\
\textbf{KAN+Euler} & $\mathbf{2.88}${\scriptsize$\pm.06$} & $\mathbf{4.22}${\scriptsize$\pm.09$} \\
\bottomrule
\end{tabular}
\end{table}

\noindent\textbf{Gains Extend to High DoFs.}
KAN+Euler attains the lowest error on both metrics (Table~\ref{tab:body}), and the GE
gain over MLP+6D is consistent across every seed. KAN+6D matches MLP+6D ($4.42$ vs.\ $4.40$),
so the improvement is attributable to the bounded native-Euler
representation, not the KAN backbone alone, mirroring the hand study in Section~\ref{sec:experiment_hand}. The advantage also grows
as the joint ranges are further bounded: scaling each joint's range by $1/\mathrm{div}$ increases
the GE improvement monotonically from $4.1\%$ ($\mathrm{div}{=}1$) to $5.2\%$
($\mathrm{div}{=}3$). The effect is gentler than in the synthetic setting because the native
coordinates are already singularity-free. A per-joint analysis shows the net gain is driven by
the axial forearm-rotation degrees of freedom, which are weakly determined by joint positions and
on which the unconstrained $6$D baseline incurs large errors; the bounded native angle
regularizes these, while on the well-determined joints the two methods are comparable. The
bounded-Euler/KAN design extends from synthetic rotations and point clouds to this high-DoF
biomechanical model, provided rotations are expressed in their native bounded coordinates.

\section{Additional Implementation Details}
\label{appx:implementation}

\subsection{Object Rotation}

We use the \emph{chair} category of ModelNet10~\cite{wu2015modelnet}, with 889 training and 100 test shapes; each shape is sampled to 512 points, centered, and scaled to unit maximum extent. Following a Siamese design, a reference cloud and a rotated copy are each encoded by a shared PointNet-style network: a per-point MLP with widths $3,64,128,1024$ and ReLU activations, max-pooled to a 1024-dimensional global feature. The two features are concatenated and decoded to the relative rotation under the chosen representation. MLP models use a two-hidden-layer head, while KAN models apply a linear projection, layer normalization, and a KAN with degree-3 splines and 10 control points. Target rotations are sampled within the bounded Euler ranges, scaled by a division factor $\mathrm{div}$. The accuracy comparison in Table~\ref{tab:pointcloud}(a) matches all models at ${\sim}1.44$M parameters at $\mathrm{div}=1.3$, whereas the efficiency study in Table~\ref{tab:pointcloud}(b) scales the MLP+6D hidden width against a compact $0.45$M KAN+Euler; the FLOPs reported there count the representation head, since the shared PointNet encoder, at roughly $286$M FLOPs, is common to all models. All models are trained with a geodesic loss on $\mathrm{SO}(3)$ using Adam with an initial learning rate of $10^{-3}$ and cosine annealing to $10^{-5}$, batch size 128, for 40{,}000 iterations, and are evaluated every 2{,}000 iterations on 5{,}000 held-out pose pairs over 3 seeds.

\subsection{Robotic Arm Articulation}

We use the Franka Emika Panda 7-DoF robotic arm, loaded from its URDF
specification via the \texttt{ikpy} library for forward kinematics (FK).
The number of active degrees of freedom (DoF) is varied from 3 to 5, activating
joints from the base outward and fixing the remaining (distal) joints at zero radians.
We generate 1M valid pose--configuration pairs by uniformly sampling
active joint angles within their mechanical limits (as specified in the URDF),
computing the end-effector (EE) pose via FK, and retaining samples within the reachable workspace.
The valid samples are split into 900K training and 100K
evaluation samples. The network input consists of the 3D EE position concatenated with the EE
rotation under the chosen representation: three ZXY Euler angles for Euler
(input dimension 6) or the first two columns of the rotation matrix for 6D
(input dimension 9).
The output is the vector of active joint angles, and all models are trained
with MSE loss in joint-angle space. All models are matched to approximately 100K parameters and trained for 30 epochs with batch size 1{,}024, using AdamW with learning rate $10^{-3}$.
Evaluation is performed every 5 epochs on a subset of 5K evaluation
samples.
For each sample, predicted joint angles are assembled into a full 7-DoF
configuration and FK is used to compute the predicted EE position.
We report the mean forward-kinematics position error in centimeters (FKE) and the success rate
under a 1\,cm FK threshold (SR@1cm).
Table~\ref{tab:franka} reports results from the best-performing epoch.
For the workspace-robustness study in Table~\ref{tab:franka}(b), we repeat the 3-DoF protocol
while scaling the axis-aligned end-effector workspace box to $0.7\times$ and $0.5\times$ its
default half-extents about the workspace center, regenerating 1M valid pose--configuration pairs
within each shrunken box and keeping all other settings unchanged.

\subsection{Biomechanical Hand Articulation}

When constructing the inverse kinematics benchmark, forward kinematics is first applied to the MANO~\cite{romero2017embodied} pose and shape parameters to generate 21 3D hand joints. All joints are expressed in a wrist-centered coordinate frame; removing the root joint results in a 60-dimensional flattened input. We exclude the root rotation from supervision and predict only the 15 articulated joints, with output dimensions of 45 and 90 for the Euler-angle and the 6D representation~\cite{zhou2019continuity}, respectively. Considering the hand joint rotation ranges defined in~\cite{zhang2024weakly}, we adopt a consistent intrinsic ZXY convention for both Euler-to-matrix and matrix-to-Euler conversions, and additionally evaluate a baseline with a YZX rotation ordering that is inconsistent with these ranges. Input joint positions are normalized to $[-1,1]$ using training data statistics. To ensure a fair comparison between MLP and KAN under matched parameter counts, the MLP uses a single hidden layer with dimensions $[60,320,45]$ (Euler) and $[60,224,90]$ (6D), while the KAN uses spline degree 3 with 4 control points and layer dimensions $[60,64,45]$ (Euler) and $[60,45,90]$ (6D). Both architectures are matched at ${\sim}34$K parameters. Models are trained using AdamW with learning rate $2\times10^{-3}$, and batch size 1024 for approximately 10K steps. For data-efficiency experiments, the training set is randomly subsampled without replacement and the number of epochs is scaled inversely with the retained data fraction to keep the total number of optimization steps approximately constant across data fractions.

\subsection{Controlled Rotation Regression}

Rotation matrices are sampled from $\mathrm{SO}(3)$ using the ZXY Euler convention. We generate 500K training and 50K test samples per configuration. The network receives a flattened $3\times 3$ rotation matrix (9 dimensions) as input
and outputs the underlying Euler angles. The gimbal-prone middle axis is sampled from $[-\pi/2,\pi/2]$ and the two outer axes from $[-\pi,\pi]$; for each division factor $\mathrm{div}$, all three ranges are uniformly scaled by $1/\mathrm{div}$ (so the middle axis occupies $[-\pi/(2\,\mathrm{div}),\pi/(2\,\mathrm{div})]$). Optimization uses Adam with batch size 1{,}024, an initial learning rate of
$1.6\times10^{-4}$ reduced to $10^{-6}$ after 10{,}000 iterations,
and a total of 31{,}250 training iterations. Both KAN and MLP models are compared under matched parameter budgets; the MLP scaling in Table~\ref{tab:acc_eff}(b) uses 4-layer MLPs with hidden widths $\{48,100,220,300\}$, while the matched KAN uses 2 hidden layers of width 16 with 12 control points. For the angle-range study in Figure~\ref{fig:angle_range}, $\mathrm{div}$ is swept from $1$ to $2$ in steps of $0.1$ and additionally over $\{3,4,8\}$, reporting the best GE achieved during training for each value.